\documentclass[letterpaper, 10pt, conference]{ieeeconf}
\IEEEoverridecommandlockouts  % needed to use the \thanks command
\usepackage{algorithm}
\usepackage{algorithmic}
\usepackage{amsmath}
\usepackage{amssymb}
\usepackage{bm}
\usepackage{color}
\usepackage{graphicx}
\usepackage{hyperref}
\usepackage[misc]{ifsym}
\usepackage{siunitx}

\newtheorem{proposition}{Proposition}
\newtheorem{corollary}{Corollary}

\newcommand{\colvec}[1]{\left[\begin{array}{c} #1 \end{array}\right]}
\newcommand{\defeq}{\stackrel{\mathrm{def}}{=}}

\renewcommand{\th}[1]{\ensuremath {#1}^{\textrm{th}}}
\renewcommand{\vec}[1]{\ensuremath \overrightarrow{#1}}

\def\eg{\emph{e.g.}~}
\def\ie{\emph{i.e.}~}
\def\etal{\emph{et al.}}

\def\pdd{\ddot{\bfp}}
\def\pd{\dot{\bfp}}

\newcommand{\bfnu}{\boldsymbol{\nu}}
\newcommand{\bfxi}{\boldsymbol{\xi}}

\newcommand{\bfa}{\ensuremath {\bm{a}}}
\newcommand{\bfb}{\ensuremath {\bm{b}}}

\newcommand{\bfg}{\ensuremath {\bm{g}}}
\newcommand{\bfh}{\ensuremath {\bm{h}}}

\newcommand{\bfn}{\ensuremath {\bm{n}}}

\newcommand{\bfp}{\ensuremath {\bm{p}}}

\newcommand{\bft}{\ensuremath {\bm{t}}}
\newcommand{\bfu}{\ensuremath {\bm{u}}}

\newcommand{\bfx}{\ensuremath {\bm{x}}}

\newcommand{\bfz}{\ensuremath {\bm{z}}}

\newcommand{\bfA}{\mathbf{A}}
\newcommand{\bfB}{\mathbf{B}}
\newcommand{\bfC}{\mathbf{C}}

\newcommand{\bfE}{\mathbf{E}}

\newcommand{\bfH}{\mathbf{H}}

\newcommand{\bfL}{\mathbf{L}}

\newcommand{\bfS}{\mathbf{S}}

\newcommand{\calC}{{\cal C}}

\newcommand{\calS}{{\cal S}}

\newcommand{\calZ}{{\cal Z}}

\newcommand{\bbR}{{\mathbb{R}}}

\usepackage{empheq}

\definecolor{bkgnd-color}{cmyk}{0.0, 0.05, 0.1, 0}
\definecolor{title-color}{cmyk}{0.1, 0.05, 0.0, 0}

\newsavebox{\mysaveboxM} % M for math
\newsavebox{\mysaveboxT} % T for text
\newcommand*\Garybox[2][Example]{%
    \sbox{\mysaveboxM}{#2}%
    \sbox{\mysaveboxT}{\fcolorbox{black}{title-color}{#1}}%
    \sbox{\mysaveboxM}{%
        \parbox[b][\ht\mysaveboxM+.5\ht\mysaveboxT+.5\dp\mysaveboxT][b]{%
        \wd\mysaveboxM}{#2}%
    }%
    \sbox{\mysaveboxM}{%
        \fcolorbox{black}{bkgnd-color}{%
            \makebox[\linewidth-5em]{\usebox{\mysaveboxM}}%
        }%
    }%
    \usebox{\mysaveboxM}%
    \makebox[0pt][r]{%
        \makebox[\wd\mysaveboxM][c]{%
            \raisebox{\ht\mysaveboxM-0.5\ht\mysaveboxT +0.5\dp\mysaveboxT-0.5\fboxrule}{\usebox{\mysaveboxT}}%
        }%
    }%
}

\def\ch{\mathrm{ch}}
\def\dTmax{\dT_\mathsf{max}}
\def\dTmin{\dT_\mathsf{min}}
\def\dT{\Delta t}
\def\sh{\mathrm{sh}}
\def\swing{\mathsf{swing}}

\title{\LARGE \bf
    Dynamic Walking over Rough Terrains by Nonlinear \\
    Predictive Control of the Floating-base Inverted Pendulum
}

\author{St\'ephane Caron$^{1}$ and Abderrahmane Kheddar$^{1,2}$%
    \thanks{*This work is supported in part by H2020 EU project COMANOID
    \url{http://www.comanoid.eu/}, RIA No 645097.}%
    \thanks{$^{1}$CNRS-UM2 LIRMM, IDH group, UMR5506, Montpellier, France.}%
    \thanks{$^{2}$CNRS-AIST Joint Robotics Laboratory (JRL), UMI3218/RL. \newline
    Corresponding author: {\tt\footnotesize stephane.caron@normalesup.org}}%
}

\begin{document}

\maketitle
\thispagestyle{empty}
\pagestyle{empty}

\begin{abstract}
    We present a real-time pattern generator for dynamic walking over rough
    terrains. Our method automatically finds step durations, a critical issue
    over rough terrains where they depend on terrain topology. To achieve this
    level of generality, we consider a Floating-base Inverted Pendulum (FIP)
    model where the center of mass can translate freely and the zero-tilting
    moment point is allowed to leave the contact surface. This model is
    equivalent to a linear inverted pendulum with variable center-of-mass
    height, but its equations of motion remain linear. Our solution then
    follows three steps: (i) we characterize the FIP contact-stability
    condition; (ii) we compute feedforward controls by solving a nonlinear
    optimization over receding-horizon FIP trajectories. Despite running at
    30~Hz in a model-predictive fashion, simulations show that the latter is
    too slow to stabilize dynamic motions. To remedy this, we (iii) linearize
    FIP feedback control into a constrained linear-quadratic regulator that
    runs at 300~Hz. We finally demonstrate our solution in simulations with a
    model of the HRP-4 humanoid robot, including noise and delays over state
    estimation and foot force control.
\end{abstract}

\section{Introduction}

A walking pattern is \emph{dynamic} when it contains single-support phases that
are not statically stable, \ie where the center of mass (COM) of the robot
leaves the area above contact and undergoes divergent dynamics. These dynamic
phases are used to increase walking speed as well as to control balance, as
illustrated by reactive stepping strategies~\cite{morisawa2009humanoids,
santacruz2013icra}. To estimate the dynamic capabilities of a rough-terrain
walking pattern generator\footnote{Pattern generators compute both feedforward
and feedback walking controls under real-time constraints, as opposed to motion
generators, which only compute feedforward controls without time constraints.}
(RT-WPG), we can measure the duration of double-support phases, or the amount
of time spent in statically-stable configurations. For instance, our previous
RT-WPG~\cite{caron2016humanoids} spends roughly 40\% of its gait in
double-support, and about 90\% of the time in statically-stable configurations.
In contrast, our present RT-WPG, in the same scenario, spends less than 5\% of
its gait in double-support, and about 40\% of the time in statically-stable
configurations.

A major concern while walking is to enforce \emph{contact stability}, \ie
making sure that contacts neither slip nor tilt while the robot pushes on them
to move. To generate contact-stable trajectories, one needs to guarantee that
all contact wrenches throughout the motion lie inside their respective wrench
cones~\cite{caron2015icra}. So far, the full problem has only been solved in
whole-body motion generation~\cite{lengagne2013ijrr}, or more recently in
centroidal motion generation~\cite{carpentier2016icra, dai2016humanoids,
ponton2016humanoids}, where computation time is provided until a solution is
found. For real-time control, it is common to reduce the number of variables by
regulating the centroidal angular momentum to $\dot{\bfL}_G = 0$. Doing so
simplifies the Newton-Euler equations of motion to:
\begin{equation*}
    \pdd_G \ = \ \lambda (\bfp_G - \bfp_Z) + \bfg,
\end{equation*}
with $\bfp_G$ the COM position, $\lambda$ a positive quantity, $\bfp_Z$ the
whole-body zero-tilting moment point (ZMP) and $\bfg$ the gravity vector. When
the COM motion is constrained to a plane, $\lambda$ is constant and we obtain
the Linear Inverted Pendulum Mode (LIPM). Predictive control of the COM in the
LIPM can be formulated as a quadratic program, where the cost function encodes
a number of desired behaviors while inequalities enforce the contact-stability
condition: the ZMP lies within the convex hull of contact points. (This
condition is actually incomplete; we will derive the complete condition below.)
This formulation successfully solved the problem of walking over flat
surfaces~\cite{kajita2003icra, wieber2006humanoids, herdt2010online}. However,
it did not extend to rough terrains where the shape of the ZMP support area
varies during motion~\cite{caron2017tro}.

Solutions for rough terrains have been proposed that track pre-defined COM
trajectories across contact switches~\cite{hirukawa2007icra, audren2014iros},
yet they were not designed to compute their own feedforward controls. In
pattern generation, three main directions have been explored so far. In one
line of work, 3D extensions of the LIPM~\cite{zhao2012humanoids,
englsberger2015tro} provided the basis for the first RT-WPGs, combining fast
footstep replanning with force-tracking control. So far, these solutions use
partial contact-stability conditions (friction is not modeled) and do not
generalize the constraint saturation behavior of their 2D
counterparts~\cite{wieber2006humanoids}. Other works went for harder nonlinear
optimization problems~\cite{naveau2017ral, serra2016humanoids,
vanheerden2017ral}, but again they did not model friction and were only applied
to walking on parallel horizontal surfaces. Finally, we recently proposed
in~\cite{caron2016humanoids} an RT-WPG that enforces full contact stability and
walks across arbitrary terrains, but at the cost of a conservative problem
linearization (a similar idea appeared simultaneously in the motion generator
from~\cite{dai2016humanoids}).

We now bridge the gap between these three directions with an RT-WPG that is (1)
based on a 3D extension of the LIPM, for which we (2) derive the full contact
stability condition and consequently (3) formulate and solve as a nonlinear
optimal-control problem. The latter being experimentally too slow for
predictive control, we further derive a constrained linear-quadratic regulator
based on the same model for high-frequency stabilization.

\section{The Floating-base Inverted Pendulum}
\label{sec:fig}

Let us consider a biped in single support. We define the surface patch $\calS$
as the convex hull of contact points, and denote by $\calC$ the contact
friction cone. In the pendulum mode, the center of pressure (COP) $C$ is
located at the intersection between $\calS$ and the central axis of the contact
wrench, which is then also a zero-moment axis. Contact breaks when this
intersection becomes empty, or switches to another mode when $C$ reaches the
boundaries of $\calS$. The Newton equation of motion of the COM in the pendulum
mode is:
\begin{equation}
    \label{ipm}
    \pdd_G \ = \ \lambda (\bfp_G - \bfp_C) + \bfg
\end{equation}
In general, this equation is bilinear as the stiffness value $\lambda$ and COP
location $\bfp_C$ are two different components of the time-varying contact
wrench. Making time explicit, the differential equation of the COM position is:
\begin{equation*}
    \pdd_G(t) - \lambda(t) \bfp_G(t) \ = \ -\lambda(t) \bfp_C(t) + \bfg
\end{equation*}
The additional constraint of the LIPM is that $\bfp_G$ and $\bfp_C$ lie in
parallel planes separated by a fixed distance $h = \bfn \cdot (\bfp_G -
\bfp_C)$, so that $\lambda = (\bfn \cdot \bfg) / h$ becomes a constant by
Equation~\eqref{ipm}. Interestingly, in the context of balance control for a 2D
inverted pendulum, Koolen~\etal~\cite{koolen2016humanoids} recently studied the
symmetric problem where $\bfp_C$ is fixed and polynomial COM interpolation
yields a variable $\lambda(t)$.

\subsection{Contact stability in single support}

\begin{proposition}
    \label{ipm-cons}
    A motion of the system~\eqref{ipm} in single contact $(\calS, \calC)$ is
    contact-stable \emph{if and only if}:
    \begin{eqnarray}
        \label{es1} \lambda & \in & \bbR^+ \\
        \label{es2} \bfp_C & \in & \calS \\
        \label{es3} \bfp_G & \in & \bfp_C + \calC
    \end{eqnarray}
    where $\calS$ and $\calC$ respectively denote the surface patch and
    friction cone of the contact.
\end{proposition}
\begin{proof}
    This result follows from injecting Equation~\eqref{ipm} into the analytical
    formula of the single-support wrench cone~\cite{caron2015icra}. See
    Appendix~\ref{proof-app} for calculations.
\end{proof}

The constraints \eqref{es1}--\eqref{es3} are expressed via set membership for a
geometric intuition. Denoting by $\bfS$ and $\bfC$ the halfspace-representation
matrices of the polygon $\calS$ and cone $\calC$, respectively, we can
formulate these constraints equivalently as:
\begin{align}
    \tag{2H} \label{es1h} -\lambda & \ \leq \ 0 \\
    \tag{3H} \label{es2h} \bfS \bfp_C & \ \leq \ 1 \\
    \tag{4H} \label{es3h} \bfC (\bfp_G - \bfp_C) & \ \leq \ 0
\end{align}
Equation~\eqref{es3} $\Leftrightarrow$ \eqref{es3h} provides the condition that
is missing in previous works~\cite{zhao2012humanoids, englsberger2015tro,
naveau2017ral, serra2016humanoids, vanheerden2017ral} to model friction. It
also completes the observation we made in Figure~6 of~\cite{caron2017tro} by
showing that, on horizontal floors, the points $\bfp_C$ where $\bfp_G \not\in
\bfp_C + \calC$ are \emph{exactly} those that need to be removed from the
convex hull of ground contact points to obtain the ZMP support area. The area
thus admits a direct geometric construction:

\begin{corollary}
    When contacts are coplanar, the ZMP support area $\calZ$ is the
    intersection between the surface patch and the backward friction cone
    rooted at the COM:
    \begin{equation}
        \calZ \ = \ \calS \cap (\bfp_G - \calC)
    \end{equation}
\end{corollary}

Proposition~\ref{ipm-cons} also gives us a geometric construction of the COM
static-equilibrium polygon in single support:

\begin{corollary}
    The COM static-equilibrium polygon in single support is either:
    \begin{itemize}
        \item empty when the friction cone $\calC$ does not contain the
            vertical $\bfg$, or
        \item equal to the vertical projection of the surface patch $\calS$
            onto a horizontal plane.
    \end{itemize}
\end{corollary}

\begin{proof}
    This result immediately follows from Proposition~\ref{ipm-cons} by
    recalling that, in static equilibrium, the COP is located at the vertical
    below the COM.
\end{proof}

Proposition~\ref{ipm-cons} shows how contact-stability inequalities, which are
bilinear in general, linearize without loss of generality when the contact
wrench is written in terms of $\lambda$ and $\bfp_C$. It is the main result we
use in this paper. Note that this linearization is not specific to single
support: in multi-contact as welll, bilinear inequalities become linear when
the contact wrench results from a fixed attractive or repellent point, as we
show in Appendix~\ref{attractors} for the interested reader. The difference
between multi-contact and single-support is that the former requires numerical
polytope projections in general~\cite{caron2016humanoids, caron2017tro}, while
we have now derived an analytical formula for the latter.

\subsection{Transferring nonlinearity}

Equations \eqref{ipm}--\eqref{es3} characterize the pendulum mode under full
contact stability:
\begin{empheq}[box={\Garybox[COP-based Inverted Pendulum]}]{align*}
    \pdd_G(t) & = \lambda(t) (\bfp_G(t) - \bfp_C(t)) + \bfg \\
    \mathrm{s.t.} & \left\{
        \begin{array}{r}
            -\lambda(t) \ \leq \ 0 \\
            \bfS \bfp_C(t) \ \leq \ 1 \\
            \bfC (\bfp_G(t) - \bfp_C(t)) \ \leq \ 0
        \end{array}
        \right.
\end{empheq}
This system is linearly constrained, but its forward equation of motion
contains a product between the two time-varying terms $\lambda$ and $\bfp_C$.
We transform it by replacing $C$ with the ZMP\footnote{Recall that all points
of the zero-moment axis $(GC)$ can be called ``zero-moment
points''~\cite{caron2017tro}.} $Z$ defined by:
\begin{equation}
    \label{transform}
    \bfp_Z \ = \ \bfp_C + \left[1 - \frac{\lambda}{\omega^2}\right] (\bfp_G
    - \bfp_C)
\end{equation}
where $\omega^2$ is a positive constant, for instance chosen as $g/\ell$ with
$\ell$ the leg length of the robot, or resulting from a COM plane
constraint~\cite{caron2017tro, zhao2012humanoids}. In the pendulum mode, this
definition coincides with the Enhanced Centroidal Moment
Pivot~\cite{englsberger2015tro}. This transformation has the benefit of making
the forward equation of motion linear:
\begin{equation}
    \label{fip}
    \pdd_G(t) \ = \ \omega^2 (\bfp_G(t) - \bfp_Z(t)) + \bfg
\end{equation}
It does not eliminate nonlinearity, however, but merely transfers it to the
system's inequality constraints.

\begin{proposition}
    \label{fip-cons}
    A motion of the system~\eqref{fip} in single contact $(\calS, \calC)$ is
    contact-stable if \emph{and only if}:
    \begin{eqnarray}
        \label{fineq1} \bfp_Z & \in & \bfp_G + \mathrm{cone}(\calS - \bfp_G) \\
        \label{fineq2} \bfp_G & \in & \bfp_Z + \calC
    \end{eqnarray}
    where $\calS$ and $\calC$ respectively denote the surface patch and
    friction cone of the contact, and $\mathrm{cone}(X)$ is the conical hull of
    a set $X$.
\end{proposition}

\begin{proof}
    We proceed by double-implication between the systems
    \eqref{ipm}--\eqref{es3} and \eqref{fip}--\eqref{fineq2}. ($\Rightarrow$)
    Rewrite Equation~\eqref{transform} as:
    \begin{equation}
        \label{rfb1}
        \bfp_Z - \bfp_G \ = \ \frac{\lambda}{\omega^2} (\bfp_C - \bfp_G)
    \end{equation}
    In this form, it is clear that \eqref{es1} $\wedge$ \eqref{es2}
    $\Rightarrow$ \eqref{fineq1}. Left-multiplying by $\bfC$, this equation
    further shows that $\bfC (\bfp_G - \bfp_Z)$ and $\bfC (\bfp_G - \bfp_C)$
    have the same sign, so that \eqref{es1} $\wedge$ \eqref{es3} $\Rightarrow$
    \eqref{fineq2}. ($\Leftarrow$) For the reciprocal implication, consider the
    inverse transform on the stiffness coefficient:
    \begin{equation}
        \label{rfb2}
        \lambda \ = \ \omega^2 \frac{\bfn \cdot (\bfp_G - \bfp_Z)}{\bfn \cdot
        (\bfp_G - \bfp_C)} \ = \ \omega^2 \frac{\bfn \cdot (\bfp_G -
        \bfp_Z)}{\bfn \cdot \bfp_G - a}
    \end{equation}
    where $\bfn \cdot \bfp = a$ is the equation of the supporting plane of the
    surface patch $\calS$. Note that $\bfn$ is both the plane normal and the
    inner axis-vector of the friction cone $\calC$. Given that the COM cannot
    be located below contact, $\bfn \cdot \bfp_G > a$ and \eqref{fineq2} imply
    that $\lambda \geq 0$. Then, Equation \eqref{rfb1} shows once again that
    $\bfC (\bfp_G - \bfp_Z)$ and $\bfC (\bfp_G - \bfp_C)$ have the same sign,
    so that \eqref{fineq2} $\Rightarrow$ \eqref{es3}. Finally, the COP being
    located at the intersection between $(GZ)$ and the supporting plane of
    $\calS$, \eqref{fineq1} implies \eqref{es2} by construction.
\end{proof}

Denoting by $\{V_i\}$ the vertices of the surface patch $\calS$, the two
conditions \eqref{fineq1}--\eqref{fineq2} can be written equivalently in
halfspace representation as:
\begin{align}
    \tag{8H} \label{fineq1h} \forall i,\ \overrightarrow{V_i V_{i+1}} \cdot (\overrightarrow{GV_i} \times \overrightarrow{V_iZ}) & \ \leq \ 0 \\
    \tag{9H} \label{fineq2h} \bfC (\bfp_G - \bfp_Z) & \ \leq \ 0
\end{align}
Using the transform~\eqref{transform}, we have therefore reformulated the
COP-based inverted pendulum into an equivalent ZMP-based model where the ZMP is
allowed to leave the surface patch. We coin it the Floating-based Inverted
Pendulum (FIP):
\begin{empheq}[box={\Garybox[Floating-base Inverted Pendulum]}]{align*}
    & \pdd_G(t) = \omega^2 (\bfp_G(t) - \bfp_Z(t)) + \bfg \\
    \mathrm{s.t.} & \left\{
        \begin{array}{r}
            \forall i,\ \overrightarrow{V_i V_{i+1}} \cdot
            (\overrightarrow{G(t)V_i} \times \overrightarrow{V_iZ(t)}) \ \leq \ 0 \\
            \bfC (\bfp_G(t) - \bfp_Z(t)) \ \leq \ 0
        \end{array}
        \right.
\end{empheq}
The forward equation of motion of the FIP is linear, as well as its friction
constraint, which will prove useful to compute feedback controls by constrained
linear-quadratic regulation. The main difficulty in computing feedforward
trajectories for this system lies in its bilinear ZMP constraint.

\begin{figure}[t]
    \centering
    \hspace{10pt}
    \includegraphics[height=5cm]{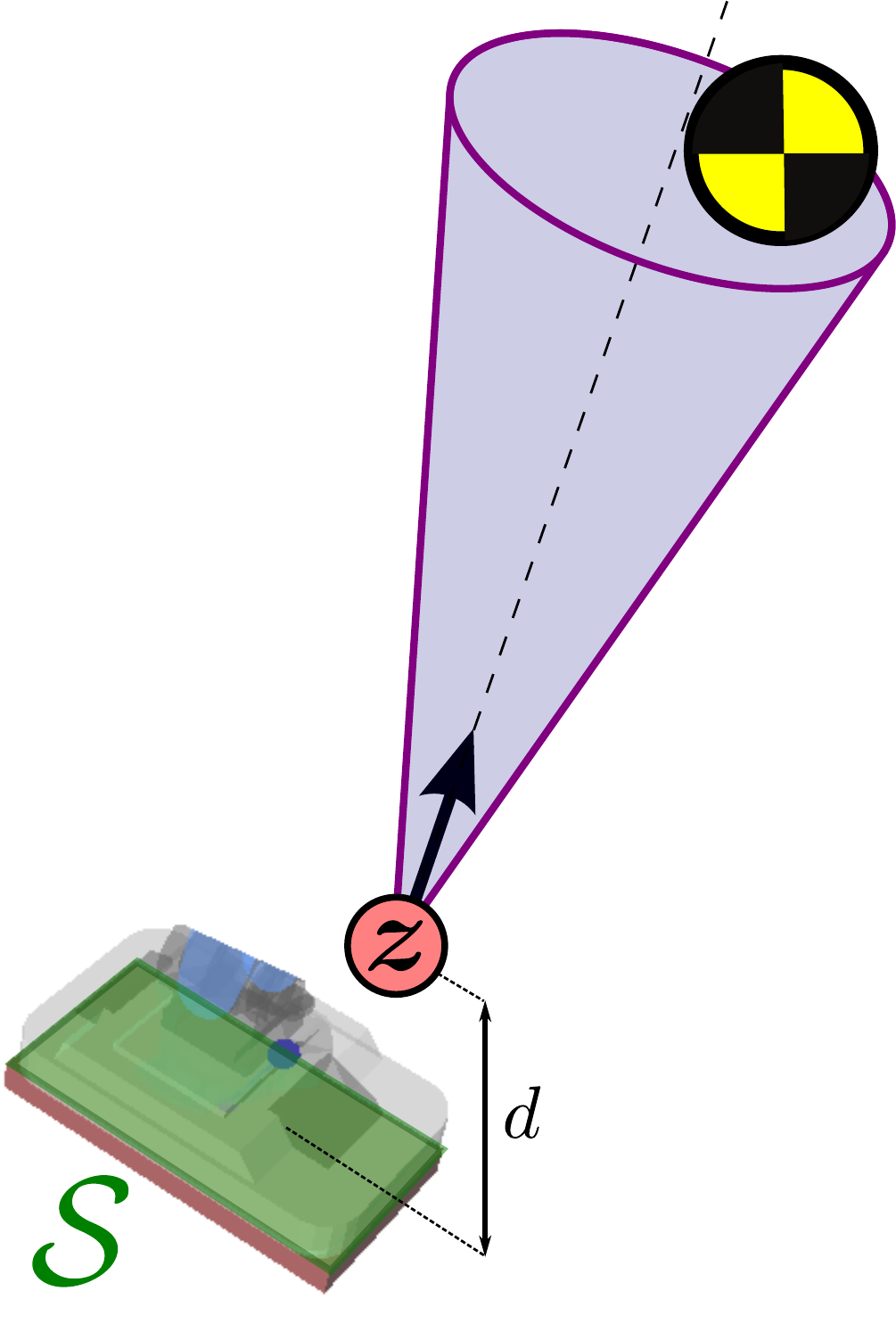}
    \includegraphics[height=5cm]{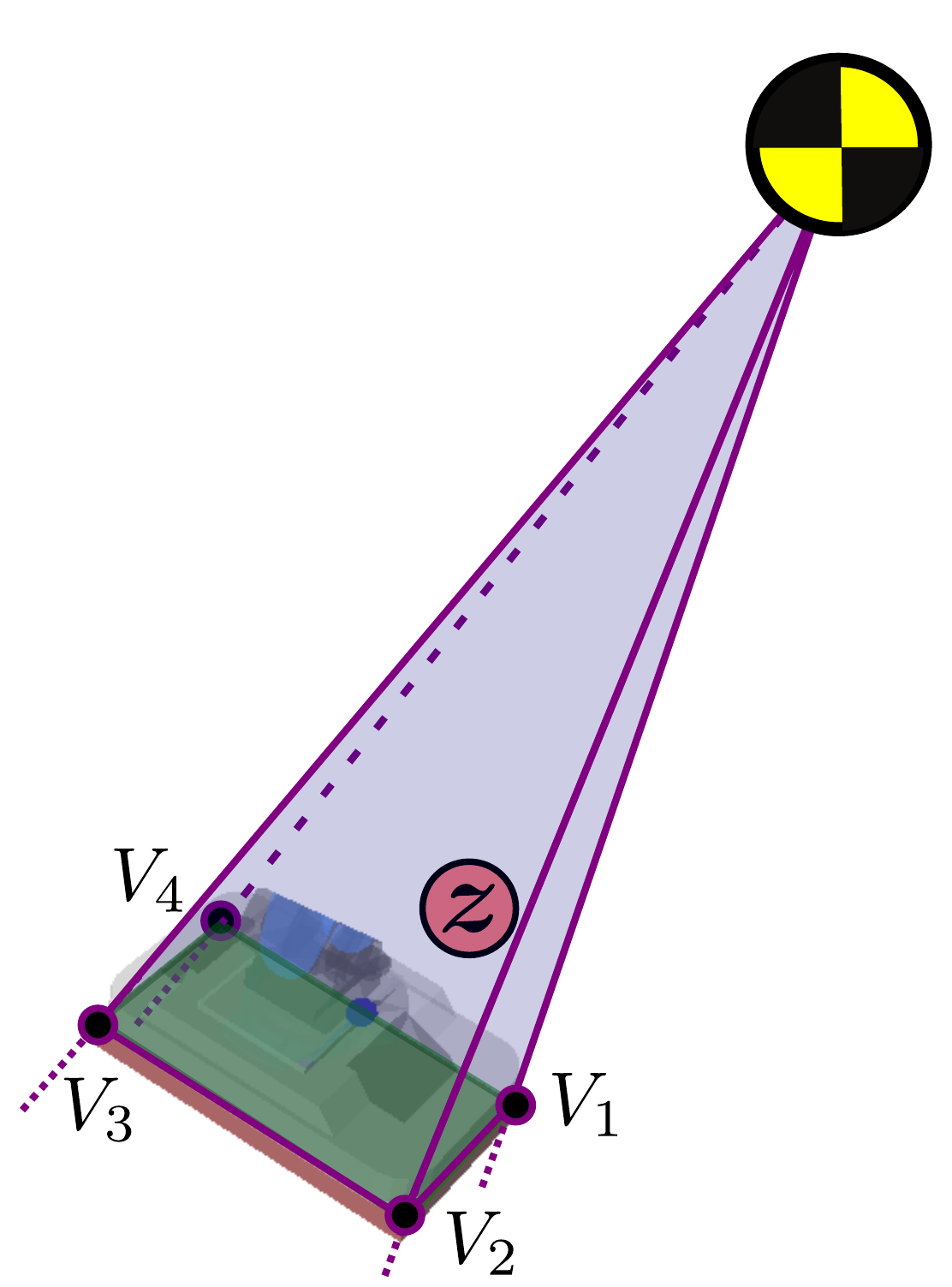}
    \hspace{10pt}
    \caption{
        The two necessary and sufficient conditions for contact stability of
        the Floating-base Inverted Pendulum. \textbf{Friction (left):} the COM
        belongs to the contact friction cone $\calC$ projected from the ZMP
        $\bfz$. \textbf{Center of pressure (right):} the ZMP $\bfz$ belongs to
        the cone projected from the COM and containing the vertices of the
        contact surface $\calS$.
    }
\end{figure}

\section{Nonlinear Predictive Control of the FIP}
\label{sec:nmpc}

We now formulate FIP predictive control as a nonlinear program (NLP). Our main
motivation in switching from convex~\cite{caron2016humanoids} to non-convex
optimization is twofold: on the one hand, solving the bilinear COP
constraint~\eqref{fineq1h} without approximation, and on the other hand,
deriving step timings as output variables rather than user-defined parameters.
This second feature is crucial over rough terrains, where proper timings depend on
terrain topology. Adapting step timings has been recently realized for walking
on horizontal floors using quadratic programming~\cite{khadiv2016humanoids},
but it had not been demonstrated yet on uneven terrains.

\subsection{Multiple shooting formulation}

As in~\cite{carpentier2016icra}, we formulate the nonlinear predictive control
(NMPC) problem by direct multiple shooting. A receding horizon over future
system states is divided into $N$ time intervals of durations $\Delta t[k]$, so
that each interval $k \in \{0, \ldots, N-1\}$ starts at time $t[k] = \sum_{j <
k} \Delta t[j]$. The variables of our NLP are:
\begin{itemize} 
    \item $\bfp_G[k]$: COM position at time $t[k]$,
    \item $\pd_G[k]$: COM velocity at time $t[k]$,
    \item $\bfp_Z[k]$: ZMP position at time $t[k]$,
    \item $\Delta t[k]$: step duration, bounded by $[\Delta t_\mathsf{min},
        \Delta t_\mathsf{max}]$.
\end{itemize}
The differential equation of the FIP~\eqref{fip} is solved over each interval
with constant ZMP located at $\bfp_Z[k]$ to obtain the matching conditions:
\begin{align}
    \bfp_G[k + 1] & \ = \ \bfp_G[k] + {\pd_G[k]} \frac{\sh[k]}{\omega} + {\bfu[k]}\frac{\ch[k] - 1}{\omega^2} \\
    \pd_G[k + 1] & \ = \ \pd_G[k] \ch[k] + \bfu[k] \frac{\sh[k]}{\omega}
\end{align}
where the following shorthands have been used:
\begin{align}
    \bfu[k] & \ = \ \omega^2 (\bfp_G[k] - \bfp_Z[k]) + \bfg \label{acc-cons} \\
    \ch[k] & \ = \ \cosh(\omega \Delta t[k]) \\
    \sh[k] & \ = \ \sinh(\omega \Delta t[k])
\end{align}
The next constraints to be enforced over at collocation points are friction and
COP inequalities~\eqref{fineq1}--\eqref{fineq2}. It is important to note here
that contact-stability is only checked at collocation times $t[k]$, as done in
the vast majority of present works that solve constrained optimal control
problems~~\cite{carpentier2016icra, dai2016humanoids, ponton2016humanoids,
wieber2006humanoids, herdt2010online, caron2017tro, audren2014iros,
naveau2017ral, vanheerden2017ral}. This does not guarantee that the constraints
will not be violated between $t[k]$ and $t[k+1]$.

Given that NLP solvers can handle both linear and nonlinear constraints, we
tried two variants of the friction cone constraint~\eqref{fineq2}:
\begin{itemize} 
    \item \textbf{FC1:} the linear constraints~\eqref{fineq2h} corresponding to
        the polyhedral approximation of the friction cone $\calC$;
    \item \textbf{FC2:} the second-order isotropic friction cone, \ie without
        approximation, which can be written as:
        \begin{equation}
        \| \overrightarrow{ZG} \|_2^2 - (1 + \mu^2) (\overrightarrow{ZG} \cdot
            \bfn) \ \leq \ 0
        \end{equation}
\end{itemize}
Second-order inequalities reduce the constraint dimension but increase its
complexity. We compared the performance of both approaches in simulations, and
observed that computations were roughly 10\% faster with FC2. This does not
mean that second-order constraints always perform better, though, as
we observed degraded performances with second-order COM-ZMP cones (obtained by
replacing the contact polygon with a contact ellipsoid).

\subsection{Boundary conditions}

We include both initial and terminal conditions in our predictive problem:
\begin{itemize} 
    \item $\bfp_G[0]$ is equal to the estimated COM position at the beginning
        of the control cycle,
    \item $\pd_G[0]$ is equal to the estimated COM velocity at the beginning of
        the control cycle,
    \item $\bfp_Z[N - 1] = \bfxi^d_\mathsf{end}$ the desired capture point at
        the end of the predictive horizon.
\end{itemize}
The ability to define capture points~\cite{pratt2006humanoids} is another
advantage of the FIP compared to models with nonlinear forward equations of
motion. From Equation~\eqref{fip}, and following the derivation from
\eg~\cite{englsberger2015tro}, the instantaneous capture point in the FIP is:
\begin{equation}
    \bfxi(t) \ = \ \bfp_G(t) + \frac{\pd_G(t)}{\omega} + \frac{\bfg}{\omega^2}
\end{equation}
The boundary value $\bfxi^d_\mathsf{end}$ is derived from the next contact
location and a reference walking velocity $v^d$ provided by the user.
Specifically, if $F$ is the last contact location of the receding horizon and
$(\bft_F, \bfb_F, \bfn_F)$ the corresponding contact frame, then:
\begin{equation}
    \bfxi^d_\mathsf{end} \ = \ \bfp_F + v^d \frac{\bft_F}{\omega} + \frac{\bfg}{\omega^2}
\end{equation}
It matches a desired COM position $\bfp_G^d$ located at the vertical above $F$,
along with a forward COM velocity equal to $v^d$.

We chose $v^d$ so that $\bfxi^d_\mathsf{end}$ belongs to the surface patch
$\calS$. This way, if the NMPC stops providing updated feedforward trajectories
for some reason (which may happen as we are solving non-convex problems), at
least regulation around the latest successful trajectory will steer the system
to a stop. One could replace this simple post-preview behavior with more
general \emph{boundedness constraints}~\cite{lanari2015humanoids}, based \eg on
heuristic post-preview ZMP trajectories derived from terrain topology.

\subsection{Cost function}

The cost function of our NLP is a weighted combination of three integral and
one terminal terms:
\begin{equation*}
    \int_{t=0}^T (w_Z \| \bfp_Z - \bfp_F \|^2 + w_G \| \pdd_G \|^2 + w_T) {\rm
    d}t + w_\xi \| \bfxi[N] - \bfxi^d_\mathsf{end} \|^2
\end{equation*}
\begin{itemize} 
    \item $\int \| \bfp_Z - \bfp_F \|^2 {\rm d}t $, where $F$ is the center of
        the supporting foot (note that there may be two different supporting
        foot in a predictive horizon due to contact switches). This term favors
        solutions with lower foot contact torques. We give it a weight of $w_Z
        = 10^{-5}$.
    \item $\int \| \pdd_G \|^2 {\rm d}t$, a regularization term used to avoid
        unnecessarily high accelerations. We give it a weight of $w_G = 10^{-3}$.
    \item $T = \int 1 {\rm d}t$, the total duration of the trajectory. This term
        plays a significant role in balancing the acceleration regularization
        term, which otherwise generates local minima where the system tries to
        avoid moving altogether. We give it a weight of $w_T = 10^{-2}$.
    \item $\| \bfxi[N] - \bfxi^d_\mathsf{end}\|^2$, where $\bfxi[N]$ is the
        capture point defined from $\bfp_G[N]$ and $\pd_G[N]$ at the end of the
        predictive horizon. This term is the state analog of the ZMP boundary
        condition. The problem is better conditioned when it is put in the cost
        function rather than as a hard constraint. We give it a weight of
        $w_\xi = 1$.
\end{itemize}

\subsection{Contact switches}

Contrary to horizontal-floor solutions that preview several future
footsteps~\cite{wieber2006humanoids, herdt2010online}, our controller previews
exactly one step ahead during single-support phases (we use our conservative
linear MPC~\cite{caron2016humanoids} during double-support phases). Our
nonlinear program does not model the swing foot trajectory. Rather, it relies
on an estimate of the time to heel strike, which we construct as
in~\cite{caron2016hal}:
\begin{itemize}
    \item Interpolate a polynomial path from the current swing-foot location to
        its target foothold (spherical linear interpolation is used to
        interpolation its orientation).
    \item Use Time-Optimal Path Parameterization~(TOPP)~\cite{pham2014tro} to
        retime the interpolated path under conservative foot acceleration
        constraints.
    \item Take the duration $T_\swing$ of the retimed path as estimate for the
        time to heel strike.
\end{itemize}
The $N$ time intervals of the receding horizon are then split into two
categories. The first half of them is dedicated to the swing interval $[0,
T_\swing]$ until heel strike, where contact stability is enforced with respect
to the current supporting foot, while for the second half it is enforced with
respect to the next foothold. This assignment is matched with step durations by
the last constraint of our NLP:
\begin{equation}
    \sum_{k=0}^{N/2} \Delta t[k] \ \geq \ T_\swing.
\end{equation}

\subsection{Implementation details}

We construct NLPs using the CasADi symbolic framework~\cite{casadi} and solve
them with the primal-dual interior-point solver IPOPT~\cite{ipopt}. Major
settings that allowed us to reach fair computation times include:
\begin{itemize} 
    \item Using \texttt{MX} rather than \texttt{SX} CasADi symbols.
    \item Using the MA27 or MA97 linear solvers within IPOPT.
    \item Capping the CPU time and number of iterations to 100 ms and 100,
        respectively. When these budgets are exceeded, the solver has most
        likely diverged away from any feasible solution.
    \item Using the adaptive rather than monotone update strategy for the
        barrier parameter, which made computations roughly 40\% faster.
\end{itemize}
With this implementation, it takes roughly 10 to 40 ms to solve an NMPC problem
(see Section~\ref{simus} for details). These computation times are on the same
scale as those reported in the warm-started phase of~\cite{carpentier2016icra},
but we don't suffer a second-long cold start to generate an initial feasible
solution. This is most likely because the problem we solve is smaller
(single-support) and we have reformulated its structure concisely.

\subsection{Tunings for variable time steps}

Each step duration $\Delta t[k]$ in the NLP is lower and upper bounded by two
parameters $\dTmin$ and $\dTmax$ that affect solver performances. Expectedly,
computation time increases with $\dTmax$, but this parameter cannot be too low
as the problem becomes infeasible below a certain threshold
$\dTmax^\mathsf{lim}$. Figure~\ref{fig:dTmax} shows how computation
times\footnote{All computations reported in this paper were run on a personal
laptop computer, CPU: Intel\textsc{(r)} Core\textsc{(tm)} i7-6500U CPU @ 2.50
Ghz.} are influenced by varying $\dTmax$ at the beginning of a feasible
single-support phase. A ``sweet range'' extends from 200 to \SI{500}{\ms}, with
the best (also riskiest) performance obtained close to the threshold. Below
this range, computation times and failure rate increase beyond usable values.
In practice, values of $\dTmax^\mathsf{lim}$ ranged between \SI{100}{\ms} and
\SI{300}{\ms} during the gait cycle, and we chose a uniform setting of $\dTmax
= \SI{350}{\ms}$.

The influence of $\dTmin$ is of a different nature. First, we note that
$\dTmin$ cannot be equal to zero in practice: all integral terms in the cost
function bias solutions toward $\Delta t[k] \to 0$, and we want to preclude
local optima where the receding-horizon duration would be zero. Setting
$\dTmin$ to the control cycle gives good performance in practice. Larger values
further improve computation times at the beginning of single-support phases,
but jeopardize convergence in the middle of the step where $T_\swing$ becomes
small. We dealt with this case by reducing $\dTmin$ on the fly when NMPC
solutions start to overshoot $T_\swing$. The first $N/2$ steps being devoted to
the swing phase, the procedure is:

    \begin{algorithmic}
        \IF{$\sum_{k=0}^{N/2} \Delta t[k] > (1 + \frac14) T_\swing$}
            \STATE $\forall k \leq N / 2,\ \dTmin[k] \leftarrow \dTmin[k] / 2$
        \ENDIF
    \end{algorithmic}

In a motion generation scenario where computation time is abundant, one could
devise global strategies such as a bisection search to tune $\dTmin$ and
$\dTmax$. The heuristic used here is rather a local parameter search spread
over control cycles. It has the benefit of incurring no additional cost.

\begin{figure}[t]
    \centering
    \includegraphics[width=\columnwidth]{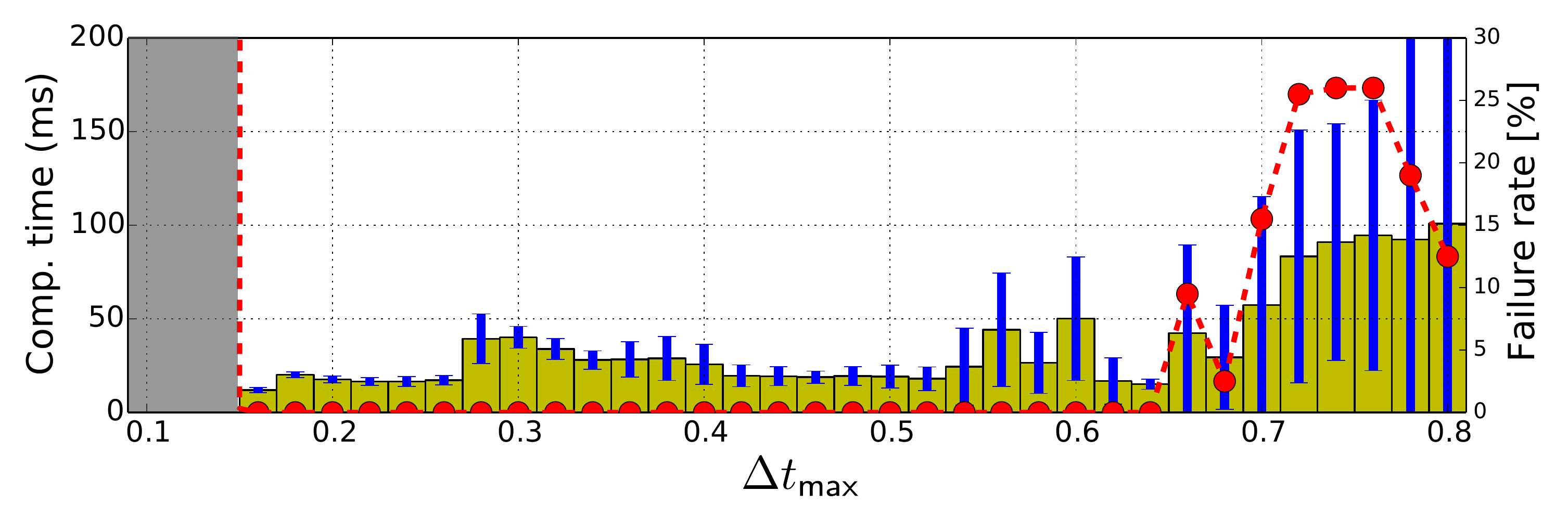}
    \caption{
        Effect of $\dTmax$ on computation times (yellow) and failure rate (red)
        of the NLP solver, for a feasible step and $N=10$ collocation points.
        Each bar includes a standard deviation estimate (blue line) computed
        over 200 runs. Above a minimum value $\dTmax^\mathsf{lim}$, the problem
        becomes feasible and all runs should ideally converge to a solution. In
        practice, computations become unstable for large values of $\dTmax$.
    }
    \label{fig:dTmax}
\end{figure}

% DISCUSSION AVEC JUSTIN

% Wrench approach is slower and less numerically stable (both are linked since
% the solver takes more iterations then) when the weight on the CAM increases.
% Furthermore, it does not converge to the same trajectory as velocity terms gets
% neglected (illustration). The approach can be interpreted as a relaxation of
% the manifold constraint $\dot{L}_G=0$ (hard constraint does not work). In case
% it does not converge, kinematic constraint (such as exponential barrier) help.
% But still...

% Barrier function on COM positions in Carpentier et al.: important otherwise
% example in pendulum mode: let COM fall (under contact) = local optimum

\subsection{Failure rate accross the gait cycle}

While computation times look promising, another metric suggests that nonlinear
optimization is not sufficient in itself to solve the NMPC problem: its failure
rate, in our case, is around 40\%. This means that, on average one out of two
to three control cycles, the solver either does not terminate, or returns a
certificate of infeasibility, or converges to an off-track solution. We found
that the terminal capture-point error $\| \bfxi[N] - \bfxi^d_\mathsf{end}\|$ is
a good indicator of this latter case, and chose to discard all solutions where
this error is above 10 cm (a liberal value close to the half-length of HRP-4's
footprint).

\section{Constrained Linear-Quadratic Regulation}
\label{sec:lqr}

To cope with the 40\% of situations where the nonlinear optimization fails to
produce a new solution on time, we design a constrained linear-quadratic
regulator (LQR) that updates the last available trajectory $\bfp_G^d, \pd_G^d,
\bfp_Z^d$ into a new feasible one starting from the current COM state. In order
to cast the regulation problem as a quadratic program, this reference
trajectory is first resampled into $M$ time steps of equal duration $\Delta T$.
Residual states and controls of the FIP are then:
\begin{align}
    \Delta \bfx[k] & \ = \ \colvec{\bfp_G[k] - \bfp_G^d[k] \\ \pd_G[k] - \pd_G^d[k]} \\
    \Delta \bfz[k] & \ = \ \bfp_Z[k] - \bfp_Z^d[k]
\end{align}
The discretized linear dynamics of these residuals are:
\begin{align}
    \label{regu0}
    \Delta \bfx[k + 1] &= \bfA \Delta \bfx[k] + \bfB \Delta \bfz[k] \\
    \bfA &= \left[
        \begin{array}{rr}
            \cosh(\omega \Delta T) \bfE & \sinh(\omega \Delta T) / \omega \bfE
            \\
            \omega \sinh(\omega \Delta T) \bfE & \cosh(\omega \Delta T) \bfE
        \end{array}
        \right] \\
    \bfB &= \left[
        \begin{array}{r}
            (1 - \cosh(\omega \Delta T)) \bfE \\
            -\omega \sinh(\omega \Delta T) \bfE
        \end{array}
        \right]
\end{align}
where $\bfE$ is the $3 \times 3$ identity matrix. Inequality constraints over
$\bfp_G$ and $\bfp_Z$ translate into similar constraints over $\Delta \bfx$ and
$\Delta \bfz$. The friction constraint~\eqref{fineq2h} becomes:
\begin{equation}
    \label{regu1}
    \left[\,\bfC\,\ \bm{0}\,\right] \Delta \bfx[k]
    - \bfC \Delta \bfz[k]
    \ \leq \ 
    \bfC (\bfp_Z^d[k] - \bfp_G^d[k])
\end{equation}
Next, define $\sigma_i[k] \defeq -\vec{V_i V_{i+1}} \cdot (\vec{G^d[k] V_i}
\times \vec{V_i Z^d[k]})$ the positive slackness of the COP constraint on the
$\th{i}$ vertex in the reference trajectory. Expanding~\eqref{fineq1h} yields,
in geometric form (we omit indexes $[k]$ to alleviate notations and write
$\Delta \bfp$ the first three coordinates of $\Delta \bfx$):
\begin{equation}
    \vec{V_i V_{i+1}} \cdot (
    \vec{G^d V_i} \times \Delta \bfz +
    \vec{V_i Z^d} \times \Delta \bfp +
    \Delta \bfz \times \Delta \bfp)
    \ \leq \  \sigma_i[k]
\end{equation}
And in matrix form:
\small
\begin{equation}
    \Delta \bfp[k]^T \bfH_i \Delta \bfz[k] + \bfh_P[k]^T \Delta \bfp[k] +
    \bfh_Z[k]^T \Delta \bfz[k]
    \leq \sigma[k]
\end{equation}
\normalsize
with $\bfH_i$ the cross-product matrix of $\vec{V_i V_{i+1}}$ and
\begin{align*}
    \bfh_P[k] & \ \defeq \ \bfH_i (\bfp^d_Z[k] - \bfp_{V_i}), \\
    \bfh_Z[k] & \ \defeq \ \bfH_i (\bfp_{V_i} - \bfp_G^d[k]).
\end{align*}
At this point, one could put polyhedral bounds on $\Delta \bfz$ or $\Delta
\bfp$ and solve a (bigger) conservative linearized system. This is \eg the
approach followed in~\cite{caron2016humanoids, dai2016humanoids} where COM
trajectories are boxed into user-defined volumes. However, contrary to these
previous works, our problem here applies to \emph{residual} variables, which we
can assume to be small. Intuitively, if $\|\Delta \bfp\| \ll \| \vec{G^d V_i}
\|$ and $\|\Delta \bfz\| \ll \| \vec{V_i Z^d} \|$, then $\|\Delta \bfp \times
\Delta \bfz\|$ should be orders of magnitude smaller than the linear term
$\|\vec{G^d V_i} \times \Delta \bfz + \vec{V_i Z^d} \times \Delta \bfp\|$. We
therefore neglect this residual cross-product, resulting in a linear COP
constraint:
\begin{equation}
    \label{regu2}
    \bfh_P[k]^T \Delta \bfp[k] + \bfh_Z[k]^T \Delta \bfz[k] \ \leq \ \sigma[k]
\end{equation}

After implementing the complete pipeline described so far, we checked the
validity of this assumption down the line. We found that, in the simulation
framework described in the next section (which includes noise and delays in
both control and state estimation) the ratio
\[
    \frac{\|\Delta \bfp \times \Delta \bfz\|}{\|\vec{G^d V_i} \times \Delta
    \bfz + \vec{V_i Z^d} \times \Delta \bfp\|}
\]
is equal on average to $0.005$ with a standard deviation of $0.005$ for 10,000
sampling times corresponding to five minutes of locomotion. That is, the
cross-product term is roughly two orders of magnitude smaller than the linear
one, which \emph{a posteriori} legitimates our assumption.

Coming back to problem formulation, our constrained LQR is finally cast as a
quadratic program with cost function:
\begin{align}
    \underset{\{\Delta \bfz[k]\}}{\mathrm{minimize}}
    \ 
    & \textstyle \sum_{k=0}^{M-1} \left(w_{xc} \| \Delta \bfx[k] \|^2 + w_z \|
    \Delta \bfz[k]\|^2\right) \nonumber \\
    & + w_{xt} \| \Delta \bfx[M] \|^2 \\
    \textrm{subject to} \ & \forall k,
    \eqref{regu0} \wedge
    \eqref{regu1} \wedge
    \eqref{regu2}
\end{align}
We solve this problem using the classical single-shooting formulation described
in \eg \cite{audren2014iros} and implemented by the
\emph{Copra}\footnote{\url{https://github.com/vsamy/Copra}} library. In
experiments, we set the terminal weight to $w_{xt}=1$ and the cumulative
weights to $w_{xc}=w_{z}=10^{-3}$.

\begin{figure}[t]
    \centering
    \includegraphics[width=\columnwidth]{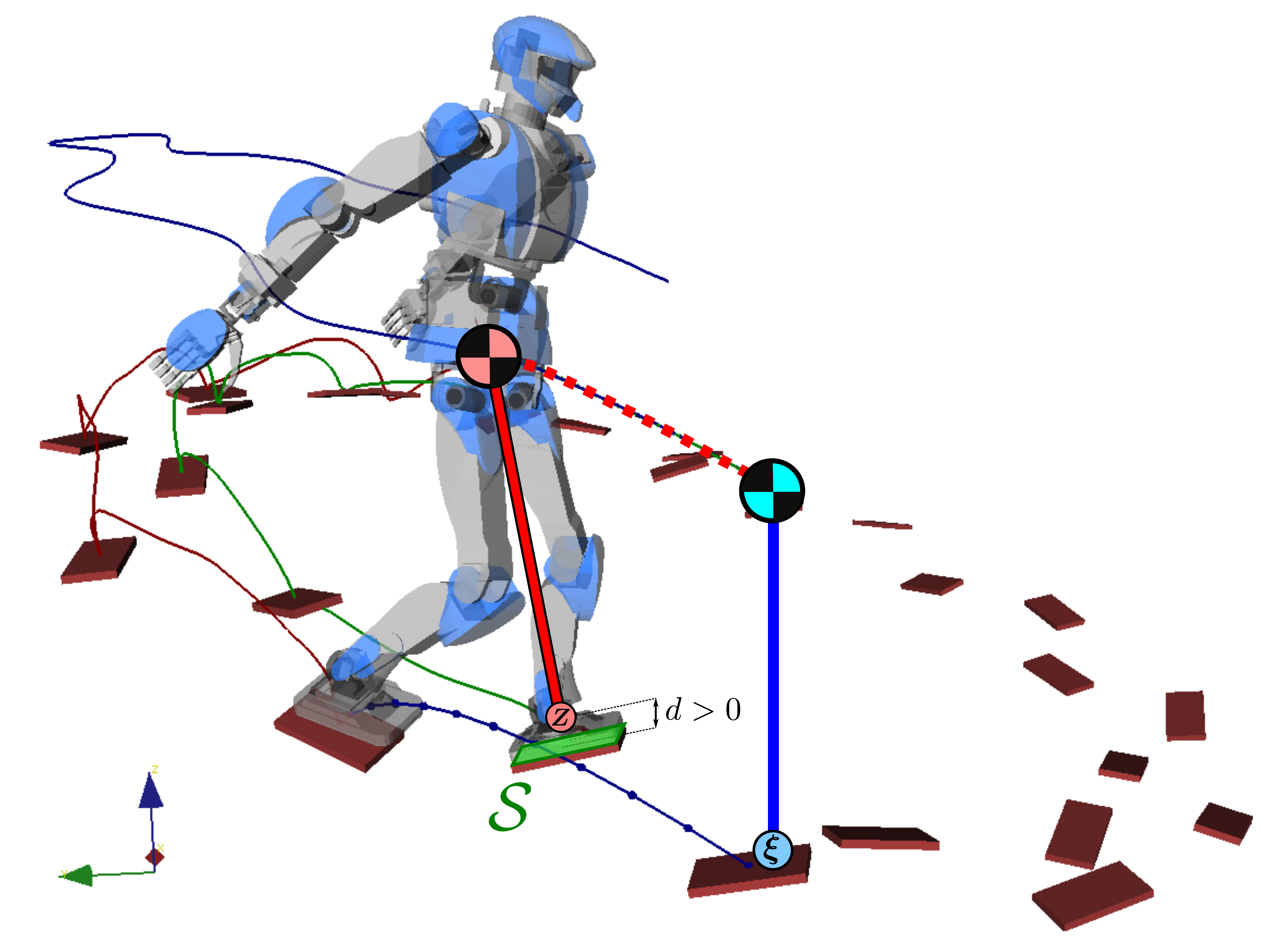}
    \caption{
        \textbf{Walking pattern generation over an elliptic staircase with tilted
        steps.} At each control cycle, a new trajectory (dotted line) is
        computed via nonlinear optimization for the Floating-base Inverted
        Pendulum (red: current state, blue: desired state at the end of the
        receding horizon). In this model, the ZMP $Z$ can leave the surface
        patch $\calS$ and the COM can move freely in 3D while keeping linear
        equations of motion.
    }
    \label{fig:downstairs}
\end{figure}

\section{Simulations}
\label{simus}

We validated the proposed method in simulation with a model of the HRP-4
humanoid robot. Our benchmark test is a randomly-generated elliptic staircase
that includes all the characteristics that we deem important for rough-terrain
locomotion: going up, forward and down using tilted contacts (no two contacts
are coplanar). Our simulations use
\emph{pymanoid}\footnote{\url{https://github.com/stephane-caron/pymanoid}}, an
extension of OpenRAVE for humanoid robotics. Compared to the results reported
in~\cite{caron2016humanoids}, these new simulations model both noise and delay
in ZMP control and COM estimation:
\begin{itemize}
    \item \textbf{COM state estimation:} zero-mean noise with amplitudes of
        \SI{10}{\cm\per\second} on position and \SI{10}{\cm\per\second\squared}
        on velocity. Nominal delay is set to \SI{20}{\ms}.
    \item \textbf{Ankle ZMP control:} zero-mean noise with amplitude
        \SI{10}{\mm\per\ms}. Control delay, \ie the characteristic duration
        before a new command is achieved, is set to \SI{20}{\ms}.
\end{itemize}

Pattern generation is supervised by a finite state machine that alternates
single and double support phases. In double support, the conservative
multi-contact controller from~\cite{caron2016humanoids} is used with the
current step target as terminal condition. When the NMPC
(Section~\ref{sec:nmpc}) running in the background finds a trajectory
traversing the next step, the state machine switches to the next swing phase
(single support).

In single support, the multi-contact controller is replaced by the LQR from
Section~\ref{sec:lqr}. When the NMPC successfully finds a new solution, usually
with a delay between 0 and 3 control cycles, this trajectory is resampled and
sent to the LQR as a new reference. The LQR then produces an updated trajectory
which is sent to the whole-body inverse kinematics and foot ZMP controllers.
The numbers of NMPC and LQR steps are set respectively to $N=10$ and $M=30$.
With this design, our pattern generator is able to locomote the humanoid
accross the elliptic staircase depicted in Figure~\ref{fig:downstairs}.
However, when disabling the linear-quadratic regulator, the robot only walks a
couple of steps before NMPC numerically unstabilities make it unable to recover
from perturbations. Sample outcomes are shown in the accompanying videos. These
results can be reproduced using the source code~\cite{code}.

One important aspect in these simulations is that they perform an independent
check of contact-wrench feasibility at every time step. Indeed, as mentioned in
Section~\ref{sec:nmpc}, constraints are only enforced at collocation points.
Optimal solutions may then violate constraints in between these points, and
additional validation is needed to make sure that this does not happen. This
point is particularly critical in our NMPC where we use a small number of
variable-duration steps, and all the more justifies the addition of an LQR with
finer discretization.

\begin{figure}[t]
    \centering
    \includegraphics[width=\columnwidth]{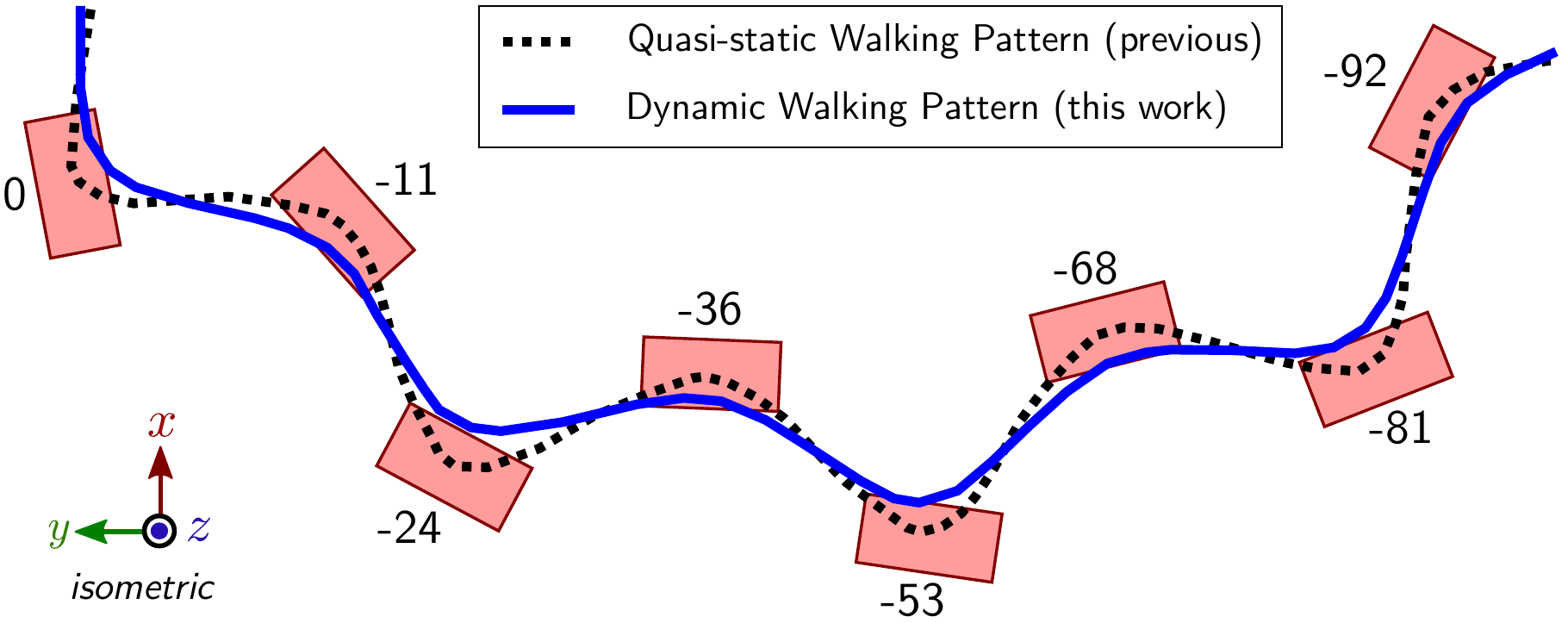}
    \caption{
        Difference in COM trajectories between this work (blue line) and our
        previous multi-contact walking pattern
        generator~\cite{caron2016humanoids} (black dotted line). Note that the
        perspective is \emph{isometric}, not linear. Footholds correspond to
        the downward part of the elliptic staircase depicted in
        Figure~\ref{fig:downstairs}. Numbers next to them indicate their
        altitude in~cm. The new trajectory is dynamic as the COM goes only
        marginally over the edges of the footholds, as opposed to the
        quasi-static one where it nears the vertical of foothold centers.
    }
    \label{fig:steps}
\end{figure}

\begin{table}[th]
    \caption{
        Performance of the NMPC and LQR controllers over two full cycles on the
        elliptic staircase.
    }
    \label{table:times}
    \centering
    \begin{tabular}{rrrc}
        Function & \# Calls & \# Successes & Time (ms) \\
        \hline
        Build NMPC &  115 &  115 & $25 \pm 8.5$ \\
        Solve NMPC & 2000 & 1452 & $21 \pm 11$ \\
        Build LQR  & 1975 & 1975 & $1.9 \pm 0.2$ \\
        Solve LQR  & 1975 & 1975 & $1.0 \pm 0.4$ \\
    \end{tabular}
\end{table}

Table~\ref{table:times} reports computation times for both the NMPC and LQR
controllers over two full cycles on the elliptic staircase. Building NMPC
problems only occurs around contact switches, the same nonlinear problem
structure being otherwise re-used between control cycles. In this scenario, the
robot called the double-support controller roughly once every two steps to
handle the extra control cycles needed by the NMPC to complete its
computations.

\section{Conclusion}

We presented a real-time rough-terrain walking pattern generator that is able
to adjust its step timings automatically. Our solution rests upon the
floating-base inverted pendulum, a model with linear equations of motion and
where contact stability can be checked using simple geometric constructions. We
developed a nonlinear predictive controller that computes feedforward walking
trajectories at roughly 30~Hz, as well as a constrained linear-quadratic
regulator computing feedback controls one order of magnitude faster. The source
code to reproduce this work is released at~\cite{code}.

\bibliographystyle{IEEEtran}
\bibliography{refs}

\appendix

\subsection{Proof of Proposition~\ref{ipm-cons}}
\label{proof-app}

In this Appendix, all coordinates are taken with respect to the local contact
frame. The analytical formula of the contact wrench cone at the origin $O$ of
this frame is given by~\cite{caron2015icra}:
\begin{equation}
    \label{W1}
    |f^x| \leq \mu f^z,\ |f^y| \leq \mu f^z
\end{equation}
\begin{equation}
    \label{W2}
    |\tau_O^x| \leq Y f^z,\ |\tau_O^y| \leq X f^z 
\end{equation}
\begin{equation}
    \label{W3}
    \tau_\mathsf{min} \leq \tau_O^z \leq \ \tau_\mathsf{max} 
\end{equation}
where the lower and upper bounds on yaw torque are:
\begin{align*}
    \tau_\mathsf{min} &\defeq -\mu (X+Y) f^z + |Y f^x - \mu \tau_O^x| + |X
    f^y - \mu \tau_O^y| \\
    \tau_\mathsf{max} &\defeq +\mu (X+Y) f^z - |Y f^x + \mu \tau_O^x| - |X
    f^y + \mu \tau_O^y|
\end{align*}
In the pendulum mode, the contact wrench is equal to:
\begin{equation*}
    \begin{split}
        f^x \ &=\ \lambda (x_G - x_C) \\
        f^y \ &=\ \lambda (y_G - y_C) \\
        f^z \ &=\ \lambda z_G 
    \end{split}
    \quad
    \quad
    \begin{split}
        \tau_O^x \ &=\ y_C f^z \\
        \tau_O^y \ &=\ -x_C f^z \\
        \tau_O^z \ &=\ x_C f^y - y_C f^x
    \end{split}
\end{equation*}
Injecting these equations into \eqref{W1}--\eqref{W3} yields two sets of
equations (we used a symbolic calculator to avoid painstaking hand calculations
here). First,
\begin{equation}
    |x_C| \leq X,\ |y_C| \leq Y \label{Z1}
\end{equation}
\begin{equation}
    |x_G - x_C| \leq \mu z_G,\ |y_G - y_C| \leq \mu z_G \label{Z2}
\end{equation}
And second, after rearranging all terms suitably:
\scriptsize
\begin{align*}
    0 & \leq (X + x_C) (\mu z_G - (y_C - y_G)) + (Y + y_C) (\mu z_G + (x_C - x_G)) \\
    0 & \leq (X + x_C) (\mu z_G + (y_C - y_G)) + (Y + y_C) (\mu z_G - (x_C - x_G)) \\
    0 & \leq (X - x_C) (\mu z_G - (y_C - y_G)) + (Y + y_C) (\mu z_G - (x_C - x_G)) \\
    0 & \leq (X - x_C) (\mu z_G + (y_C - y_G)) + (Y + y_C) (\mu z_G + (x_C - x_G)) \\
    0 & \leq (X - x_C) (\mu z_G + (y_C - y_G)) + (Y - y_C) (\mu z_G - (x_C - x_G)) \\
    0 & \leq (X + x_C) (\mu z_G - (y_C - y_G)) + (Y - y_C) (\mu z_G - (x_C - x_G)) \\
    0 & \leq (X + x_C) (\mu z_G + (y_C - y_G)) + (Y - y_C) (\mu z_G + (x_C - x_G)) \\
    0 & \leq (X - x_C) (\mu z_G - (y_C - y_G)) + (Y - y_C) (\mu z_G + (x_C - x_G))
\end{align*}
\normalsize
All right-hand side terms in this second set can be writen as $a b + c d$,
where $a,b,c,d$ are positive slackness variables from the first set of
inequalities~\eqref{Z1}--\eqref{Z2}. Therefore, all constraints in the second
set are redundant, and the contact wrench cone in irreducible form is given
by~\eqref{Z1}--\eqref{Z2}. We conclude by noting how \eqref{Z1} corresponds to
$\bfp_C \in \calS$ while \eqref{Z2} represents $\bfp_G \in \bfp_C + \calC$.

\subsection{Support Volumes for Virtual Repulsors and Attractors}
\label{attractors}

Let $\bfA_O$ denote the inequality matrix of the contact wrench cone taken with
respect to a fixed point $O$. In the pendulum mode, contact stability can be
written~\cite{caron2016humanoids} in terms of the position and acceleration of
the COM as
\begin{equation}
    \label{cstabi}
    (\bfa + \bfa_O \times \bfp_G) \cdot (\pdd_G - \bfg) \ \leq \ 0
\end{equation}
over all rows $(\bfa, \bfa_O)$ of the inequality matrix $\bfA_O$. These
expressions are bilinear and not positive-semidefinite in general, which
precludes their direct use with \eg convex optimization. There is however one
interesting setting where these inequalities linearize without loss of
generality:
\begin{proposition}
    If the COM control law follows a proportional attractor or repulsor $H$
    with stiffness $k \in \mathbb{R}$, that is
    \begin{equation}
        \label{law}
        \pdd_G \ = \ k (\bfp_H - \bfp_G),
    \end{equation}
    then the set of contact-stable positions $\bfp_G$ is a polyhedral cone
    rooted at the apex $\bfnu := \bfp_H - \bfg / k$.
\end{proposition}
\begin{proof}
    Injecting the control law~\eqref{law} into \eqref{cstabi} yields:
    \begin{align}
        k (\bfa + \bfa_O \times \bfp_G) \cdot (\bfp_H - \bfp_G - \bfg / k)
        & \ \leq \ 0
    \end{align}
    Defining $\bfnu := \bfp_H - \bfg / k$, this expression expands to:
    \begin{align}
        -k (\bfa + \bfa_O \times (\bfp_G - \bfnu + \bfnu)) \cdot (\bfp_G -
        \bfnu) & \ \leq \ 0 \\
        \label{Pineq} -k (\bfa + \bfa_O \times \bfnu) \cdot (\bfp_G - \bfnu) &
        \ \leq \ 0
    \end{align}
    using the fact that $(\bfa \times \bfb) \cdot \bfb = 0$ (the scalar triple
    product is a Gram determinant). We recognize the expression of a linear
    polyhedral cone in $\bfp_G$ with apex $\bfnu$.
\end{proof}

This property applies to the following two cases:
\begin{itemize}
    \item \textbf{Virtual Repellent Point:} $k = -\omega^2 < 0$ and $H$ is the
        VRP defined by Englsberger et al.~\cite{englsberger2015tro}. Then,
        Equation~\eqref{Pineq} defines the cone $\calC_\mathsf{VRP}(H,
        \omega^2)$ of sustainable COM positions when the VRP is located at
        $\bfp_H$.
    \item \textbf{Virtual Attractive Point:} $k > 0$ and $\bfp_H = \bfp_G^d$ is a
        desired COM location. In this case, Equation~\eqref{Pineq} defines the
        cone of COM positions that can be steered toward $\bfp_G^d$ for a given
        stiffness $k$.
\end{itemize}
Stabilizing the COM around a reference position $\bfp_G^d$ requires variable
VRPs in the first approach and variable stiffness with the second one.

\end{document}